\documentclass[letterpaper, 10 pt, conference]{ieeeconf}  % Comment this line out if you need a4paper

\IEEEoverridecommandlockouts                              % This command is only needed if 
\pdfoutput=1
\usepackage{amsmath,amssymb,amsfonts}
\usepackage{xcolor}
\usepackage[]{graphicx}
\usepackage{wrapfig}
\usepackage{multirow}
\usepackage[caption=false,font=footnotesize]{subfig}
\usepackage{etoolbox}
\usepackage{myMacros}

\newtheorem{theorem}{Theorem}

\newtoggle{IROS}
\toggletrue{IROS}
\title{\LARGE \bf
Learning Deep Energy Shaping Policies for Stability-Guaranteed Manipulation
}

\author{Shahbaz A. Khader$^{1,2}$, Hang Yin$^{1}$, Pietro Falco$^{2}$ and Danica Kragic$^{1}$% 
\thanks{A supplementary video can be found at https://youtu.be/5iwF-\_Ecuag.}
\thanks{$^{1}$Robotics, Perception, and Learning lab, Royal Institute of Technology, Sweden.
{\tt\small \{shahak, hyin, dani\}@kth.se}.}%
\thanks{$^{2}$ASEA Brown Boveri (ABB) Corporate Research, Sweden.
{\tt\small pietro.falco@se.abb.com}.}%
\thanks{This work was partially supported by the Wallenberg AI, Autonomous Systems and Software Program (WASP) funded by the Knut and Alice Wallenberg Foundation.}
}

\begin{document}

\maketitle
\thispagestyle{empty}
\pagestyle{empty}

%%%%%%%%%%%%%%%%%%%%%%%%%%%%%%%%%%%%%%%%%%%%%%%%%%%%%%%%%%%%%%%%%%%%%%%%%%%%%%%%
\begin{abstract}
Deep reinforcement learning (DRL) has been successfully used to solve various robotic manipulation tasks. However, most of the existing works do not address the issue of control stability. This is in sharp contrast to the control theory community where the well-established norm is to prove stability whenever a control law is synthesized. What makes traditional stability analysis difficult for DRL are the uninterpretable nature of the neural network policies and unknown system dynamics. In this work, stability is obtained by deriving an interpretable deep policy structure based on the \textit{energy shaping} control of Lagrangian systems. Then, stability during physical interaction with an unknown environment is established based on \textit{passivity}. The result is a stability guaranteeing DRL in a model-free framework that is general enough for contact-rich manipulation tasks. With an experiment on a peg-in-hole task, we demonstrate, to the best of our knowledge, the first DRL with stability guarantee on a real robotic manipulator.

% Since Lyapunov stability is established in a deterministic setting, parameter space exploration is adopted in order to preserve the stability property. Furthermore,   
% that is derived from \textit{energy shaping} control of Lagrangian systems. We show that the ES policy

% We contribute towards achieving stability guarantees in DRL of robotic manipulation that may even involve physical interaction with an unknown environment. 

%  The proposed \textit{energy shaping} policy, despite being a deep neural network, has unconditional stability property that can be  preserved by parameter space search methods. Our experiments, which include a real-world peg insertion task by a 7-DOF robot, validate the proposed policy structure and demonstrate the benefits of stability in RL.

\end{abstract}

%%%%%%%%%%%%%%%%%%%%%%%%%%%%%%%%%%%%%%%%%%%%%%%%%%%%%%%%%%%%%%%%%%%%%%%%%%%%%%%%
\section{Introduction}
\label{sct:intro}

Deep reinforcement learning (DRL) has emerged as a promising approach for autonomously acquiring manipulation skills, even in the form of complex end-to-end visuomotor policies \cite{levine2016end}. In the context of manipulation, where contact with the environment is inevitable, data-driven methods such as DRL that can automatically cope with complex contact dynamics are very appealing. DRL derives its power from expressive deep policies that can represent complex control laws. However, one aspect that is often forgotten is control stability. Although recent advances in stability related DRL \cite{berkenkamp2017safe,controlbarrier_aaai2019,Choi-RSS-20,variancereduct_icml2019} is encouraging, all of them are model-based approaches that either learn dynamics or have strong assumptions about it. Naturally, these methods would struggle in the context of robotic manipulation where highly nonlinear and possibly discontinuous contact dynamics \cite{modellearnRAL2020khader} are common.

\begin{figure}[t]
\centering
\includegraphics[width=0.48\textwidth]{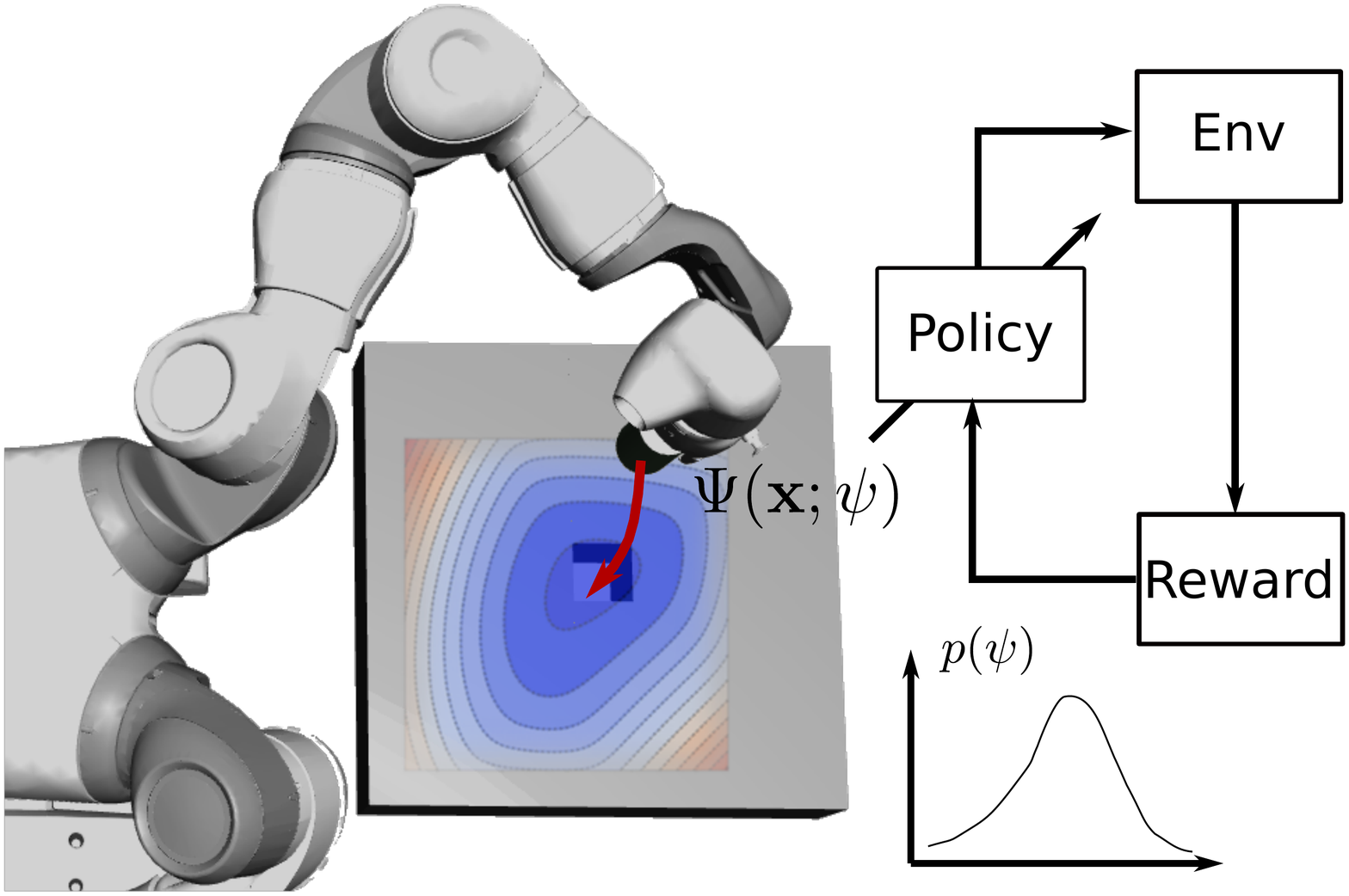}
\caption{Learning neural policies with parameterized Lagrangian energy functions; ensuring stability guarantee via exploring in the parameter space.}
\label{fig:normflow_summary}
\end{figure}

We focus on motion generation for a serial-link manipulator that already established a rigid grasp on an object. This is a well-studied problem \cite{siciliano2010robotics}, where stability is analyzed according to the Lyapunov method \cite{slotine1991applied}. Stable RL is challenging because traditional Lyapunov analysis assumes known and interpretable forms of dynamics and control law, both of which are unavailable in the DRL setting. We approach the problem by asking the question: can Lyapunov stability be guaranteed by virtue of the policy structure alone, irrespective of the robot-environment interaction and the policy parameter values? If so, stable RL can be achieved in a model-free framework that avoids learning complex contact dynamics and requires only unconstrained policy search algorithms. In our previous works \cite{stableRAL2020khader,khader2020learning}, we addressed this question by leveraging the stability property of a \textit{passive} manipulator interacting with a \textit{passive} environment, even when the environment is unknown. However, none of these methods featured policies that were fully parameterized as a deep neural network, thus potentially limiting the expressivity of the learned policies. 

In this paper, we follow \cite{stableRAL2020khader} and \cite{khader2020learning} and formulate a model-free stable RL problem for learning manipulation skills that may involve physical contact with the environment. The main contribution is the parameterization of a deep neural policy, the structure of which alone guarantees stability without the need for learning dynamics or constrained policy search. The proposed policy, called the \textit{energy shaping} (ES) policy, is derived from the first principles of Lagrangian mechanics. It consists of a position dependent convex potential function and a velocity dependent damping function; while the former is represented as an Input Convex Neural Network (ICNN) \cite{amos2017input}, the latter is a specially tailored fully connected network. We show through Lyapunov analysis that stability is guaranteed for any value of network parameters and also while interacting with any unknown but passive environment. Parameter space exploration is adopted, in the form of Cross Entropy Method (CEM)~\cite{cem-rubinstein-2003}, since the deterministic Lyapunov analysis adopted in this work allows only a deterministic policy. Our experiments, which include a real-world peg insertion task by a 7-DOF robot, validate the proposed policy structure and demonstrate the benefits of stability in RL.

% second, the representational power of the deep networks can encode non trivial motion paths and physical interaction behaviors. 

%  can also be seen as an instance of using physics-based prior for structuring interpretable policies.

\section{Related Work}
Stability has been recently considered in several RL related works. The analysis in \cite{berkenkamp2017safe} relies on smoothness properties of the learned dynamics model and, therefore, would not cope with contact-rich tasks. Another method \cite{controlbarrier_aaai2019} assumes a model prior and learns only the unactuated part of the dynamics. The work in \cite{variancereduct_icml2019} assumes an existing prior stabilizing controller while \cite{Choi-RSS-20} uses Control Barrier Function in a robust control setting. Methods such as \cite{controlbarrier_aaai2019, variancereduct_icml2019, Choi-RSS-20} require good nominal dynamics models and can only tolerate disturbances to some extent. Contact dynamics can neither be easily modelled, even nominally, nor can it be treated as a disturbance due to its magnitude and persistence. Our method depends only on a prior gravity model and neither does it make any assumptions on the contact dynamics--except that the environment is passive--nor does it learn any. 

Recently, \cite{hanstableactor2020} proposed a model-free stability guaranteed RL but it is limited to either a reward function that is the norm of the state, or to partial stability. Our method does not impose any limitations on the reward function and is based on the standard notion of Lyapunov stability. Like \cite{stableRAL2020khader} and \cite{khader2020learning}, our method also considers the stability of manipulators interacting with passive environments. Given the gravity compensation model, these methods can be considered as model-free RL. The method in \cite{stableRAL2020khader} features an analytic policy with arguably limited expressivity. The policy in \cite{khader2020learning} has part analytic structure that is not learned and part deep network that is learned. In this work, we parameterize the policy almost entirely as a deep neural network and all parts of it are learned. The policy structure in \cite{khader2020learning} is comparable to the proposed method since it also aims for stability by virtue of the policy structure alone.

Many recent methods have used the Lagrangian formulation~\cite{lutter2019iclr,zhong2020unsupervised} and also the closely related Hamiltonian formulation~\cite{greydanus2019hamiltonian,toth2020iclr}, as physics-based priors for learning dynamics models. In contrast, our method focuses on model-free reinforcement learning and is influenced by the Lagrangian physics-based prior via the principle of energy shaping control. Interestingly, \cite{zhong2020unsupervised} uses a manually designed energy shaping control while leaving policy learning to future work. Spragers et al. \cite{sprangers2014reinforcement} employ energy shaping control within an RL framework but, unlike our case, is based on \textit{port-Hamiltonian} formulation. Furthermore, the work is limited to linear-in-parameters basis function models instead of expressive deep models. The general form of the energy shaping control was also used in \cite{Khansari-Zadeh2017} for learning from demonstration, albeit without any neural network policy or reference to energy shaping of Lagrangian systems.

Parameter space exploration in the form of Evolution Strategies has recently emerged as a viable alternative to action space exploration in RL \cite{salimans2017evolution,Plappert18Parameter}. A relevant benefit for robotics is smooth trajectories during rollout \cite{ruckstiess2008state}. The simplest method in this class is the CEM method which has been previously used for policy search in \cite{mannor2003cross} and even stability-guaranteed policy search in \cite{stableRAL2020khader}. In this work, we use CEM for parameter space exploration.

%%%%%%%%%%%%%%%%%%%%%%%%%%%%%%%%%%%%%%%%%%%%%%%%%%%%%%%%%%%%%%%%%%%%%%%%%%%%%%%%

\section{Preliminaries}
\label{sct:preliminaries}

\subsection{Reinforcement Learning}
\label{sct:prelim:rl}
Reinforcement learning is a data-driven approach for solving Markov decision processes (MDP) that are defined by an environment (or dynamics) model $p(\mathbf{s}_{t+1}|\mathbf{s}_{t},\mathbf{a}_{t})$, an action space $\mathbf{a}\in\mathcal{A}$, a state space $\mathbf{s}\in\mathcal{S}$ and a reward function $r(\mathbf{s}_t, \mathbf{a}_t)$. The environment is usually assumed to be unknown. In the episodic setting, the goal is to find the optimal policy $\pi_{\theta}(\mathbf{a}_t|\mathbf{s}_t)$ of the agent, that acts on the environment through the action space, that maximizes the expected sum of reward for a desired episode length $T$ 
\begin{equation}
\label{eqn:rl_opt}
    \theta^* = \underset{\theta}{\operatorname{argmax}} J(\theta);\ J(\theta) =\mathbb{E}_{\mathbf{s}_{0}, \mathbf{a}_{0}, ..., \mathbf{s}_{T}}[\sum_{t=0}^Tr(\mathbf{s}_t, \mathbf{a}_t)],
\end{equation}
Here $\{\mathbf{s}_{0}, \mathbf{a}_{0}, ..., \mathbf{s}_{T}\}$ is a sample trajectory (episode) from the distribution induced in the stochastic system.

Our formulation differs from the above since we choose the policy and the dynamics to be deterministic:  $\mathbf{a}_t=\pi_{\theta}(\mathbf{s}_t)$ and $\mathbf{s}_{t+1}=f(\mathbf{s}_t,\mathbf{a}_t)$, respectively. A trajectory distribution is still induced as a result of maintaining a distribution over policy parameters $p(\theta)$. Adapting $p(\theta)$ until it reaches the solution is called parameter space exploration~\cite{Plappert18Parameter} \cite{ruckstiess2008state}. This readily allows exploration for deterministic policies and also results in smooth trajectories \cite{ruckstiess2008state}.

\subsection{Cross Entropy Method}
\label{sct:prelim:cem}
Cross Entropy Method (CEM)~\cite{cem-rubinstein-2003} is a black-box optimization method. In the context of the RL problem, it uses a sampling distribution $q(\theta|\mathbf{\Phi})$ to generate samples $\{\theta_n \}_{n=1}^{N_s}$ and evaluates the fitness of each of them according to the reward quantity $J(\theta_n)$. The fitness information is then used to compute updates for the sampling distribution parameters $\mathbf{\Phi}$. The iterative process can formalized as:
\begin{equation}\label{eqn:cem}
    \mathbf{\Phi}^{i+1} = \underset{\mathbf{\Phi}}{\operatorname{argmax}}\sum\limits_{n}\mathbb{I}(J(\theta_n^i)) \log q(\theta_n^i|\mathbf{\Phi})   \quad     \theta_n^i \sim q(\theta|\mathbf{\Phi}^i)
\end{equation}
where the indicator function $\mathbb{I}(\cdot)$ selects only the best $N_e$ samples or \textit{elites} over all $N_s$ samples based on individual fitness $J(\theta_n)$. The update of $\mathbf{\Phi}^{i+1}$ from the samples $\{\theta_n^i \}_{n=1}^{N_e}$ is performed by maximum likelihood estimation (MLE), which is simple and straightforward for a Gaussian $q(\theta | \mathbf{\Phi})$. The index $i$ directly corresponds to iterations in RL and eventually $q(\theta|\mathbf{\Phi}^i)$ converges to a narrow solution distribution with high fitness and thus solving Eq. \eqref{eqn:rl_opt} approximately.

\subsection{Stability in RL}\label{sec:bckgrnd_stab}
Lyapunov stability is the main tool used to study stability of controlled nonlinear systems such as robotic manipulators \cite{slotine1991applied}. The system is said to be stable at an equilibrium point if the state trajectories that start close to it remain bounded in state space and asymptotically stable if they also converge to it. Such properties are necessary for any controlled system to be certified safe and predictable. To prove Lyapunov stability, an appropriate Lyapunov function $V(\mathbf{s})$--a positive definite, radially unbounded and scalar function of the state $\mathbf{s}$--has to be shown to exist with properties $V(\vec{s})$ is in $C^1$, $V(\vec{s}^*)=0$ and $\dot{V}(\vec{s})<0\ \forall \vec{s}\ne\vec{s}^*$ where $\vec{s}^*$ is the equilibrium point. Although stability analysis requires knowledge of the system dynamics, it is still possible to reason about the stability of unknown interacting systems if they can be shown to be passive. Passivity means that a system can only consume energy but not generate energy, or $\dot{V}(\vec{s})<=P_i$ where $P_i$ is the power that flows into the system.

A stable RL policy can guarantee all rollouts to be bounded in state space and eventually converge to the goal position demanded by the task. If the policy ensures that the goal position is the only equilibrium point, we can refer to the stability of the system instead of an equilibrium point. 

% The challenge, however, is to guarantee stability in spite of incomplete knowledge of dynamics, partially trained or untrained deep policy and random exploration.

\subsection{Lagrangian Mechanics} \label{sct:prelim:lagrangian_mechanics}
Lagrangian mechanics can be used to describe motions of physical systems in any generalized coordinates. A robotic manipulator arm is no exception and its motion can be represented by the Euler-Lagrange equation,
\begin{align}
    \diff \left(\pardiff[\vxd]L\right)-\pardiff[\vx]L=\vu; \quad L(\vx,\vxd)=T(\vxd)-V(\vx) \label{eqn:lagrangian}
\end{align}
where $T$ and $V$ are the kinetic and potential energies, respectively, the quantity $L$ is the Lagrangian, $\vx$ is the generalized coordinates and $\vu$ is the generalized non-conservative force acting on the system. $V$ represents the gravitational potential function with its natural equilibrium point $\vx^*$, one with the least energy. For a manipulator under the influence of an external force $\vec{f}_{ext}$, Eq. (\ref{eqn:lagrangian}) can be expanded to the standard manipulator equation,
\begin{align}
    \mathbf{M}(\vx)\Ddot{\vx} + \mathbf{C}(\vx, \vxd)\vxd + \mathbf{g}(\vx) = \mathbf{u} + \vec{f}_{ext}, \label{eqn:manipultor}
\end{align}
where $\mathbf{M}, \mathbf{C}$ and $\mathbf{g}$ are the inertia matrix, the Coriolis matrix and the gravitational force vector, respectively. 

\subsection{Energy Shaping Control of Lagrangian Systems} \label{sct:prelim:energy_based_control}
The natural equilibrium point $\vx^*$ may not be our control objective, but instead it may be another desired goal $\vx_ d^*$. Furthermore, we may also be interested in shaping the trajectory while moving towards $\vx_ d^*$. A general control framework that incorporates both these aspects is \textit{energy shaping}~\cite{ortega2001putting} control. The energy shaping control law is of the form $\vu = \mathbf{\beta}(\vx)+\vec{v}(\vxd)$, where $\mathbf{\beta}(\vx)$ is the \textit{potential energy shaping} term and $\vec{v}(\vxd)$ is the \textit{damping injection} term. Potential energy shaping proposes to cancel out the natural potential function $V$ and replace it with a desired one $V_d$ such that not only is the new equilibrium point placed exactly at $\vx_ d^*$ but it also helps render the desired trajectory. The damping injection term's purpose is to bring the system to rest at $\vx_ d^*$ by avoiding perpetual oscillations. Energy shaping is as follows:
\begin{align}
    &\diff \left(\pardiff[\vxd]L\right)-\pardiff[\vx]L=\mathbf{\beta}(\vx)+\vec{v}(\vxd)\label{eqn:energy_control}\\
    &\mathbf{\beta}(\vx)=\pardiff[\vx]V(\vx)-\pardiff[\vx]V_d(\vx);\quad \vec{v}(\vxd)=-\mathbf{D}(\vxd)\vxd\label{eqn:potential_shaping}
    % & \label{eqn:damping_inj}
\end{align}
where $\mathbf{D}(\vxd)\succ\vec{0}$ is a damping matrix function. Note that $\vec{v}(\vxd)=-\mathbf{D}(\vxd)\vxd$ is a particular choice for achieving $\vxd^T\vec{v}(\vxd)<0$ or dissipation. Let $V_d(\vx)$ be convex with the unique minimum at $\vx_d^*$. Then the controller
\begin{align}\label{eqn:energy_shaping_controller}
    % \vu = \mathbf{\beta}(\vx)+\vec{v}(\vxd) = -\pardiff[\vx]V(\vx)-\pardiff[q]V_d(\vx)-\mathbf{D}(\vxd)\vxd
     \vu = \pardiff[\vx]V(\vx) -\pardiff[\vx]V_d(\vx)-\mathbf{D}(\vxd)\vxd
\end{align}
can potentially shape the trajectory to any arbitrary $\vx_d^*$ if appropriate $V_d(\vx)$ and $\mathbf{D}(\vxd)$ can be synthesized. The first term plays no role because it is only the gravity compensation. Unfortunately, no general solution exists for such synthesis except for some trivial cases. The benefit of this approach is that, using a Lyapunov function defined as $V_d(\vx)+T(\vxd)$, closed-loop stability and passivity can be easily achieved for any $V_d(\vx)$ and $\mathbf{D}(\vxd)$. See Sec. \ref{sct:methods:stability_proof}.

\subsection{Input Convex Neural Network (ICNN)}
\label{sct:prelim:icnn}
Fully connected ICNN or FICNN, a version of ICNN \cite{amos2017input}, has the architecture for a layer $i=0,...,k-1$
\begin{align}
    \vec{z}_{i+1} = g_i(W_i^{(\vec{z})}\vec{z}_i + W_i^{(\vec{x})}\vx + b_i)\label{eqn:ficnn}
\end{align}
where $\vec{z}_i$ represents layer activations and $\vec{z}_0, W_0^{(\vec{z})}\equiv\vec{0}$. FICNN is designed to be a scalar-valued function that is convex in its input, hence the name ICNN. Convexity is guaranteed if $W_i^{(\vec{z})}$ are non-negative and the non-linear activation functions $g_i$ are convex and non-decreasing. The remaining parameters $W_i^{(\vec{x})}$ and $b_i$ have no constraints. The network input can be any general quantity although it will only be the position coordinate in our formulation.

\section{Methodology}
\label{sct:methods}

% In regard to the first question above, we propose a novel solution for structuring and parameterizing a stability-guaranteeing deep policy. The policy would retain its stability property even when the manipulator is physically interacting with an unknown but \textit{passive} (see Sec. \ref{sct:methods:stability_proof}) environment. 
Our goal is stability-guaranteed DRL for manipulation tasks that may involve contact with the environment. Furthermore, we would like to obtain stability by virtue of the policy structure alone so that an unconstrained policy search within a model-free framework is possible.

\subsection{Structuring Policies from Physics Prior}\label{sct:policy_param}
It was shown in Sec. \ref{sct:prelim:energy_based_control} that the control form in Eq. (\ref{eqn:energy_shaping_controller}), comprising of a convex potential function $V_d(\vx)$ with a unique minimum at $\vx_d^*$ and also a damping term $\mathbf{D}(\vxd)\succ\vec{0}$, can generate a trajectory to any arbitrary $\vx_d^*$. We shall consider $\vx_d^*=\vec{0}$ without loss of generality. As a first step, we drop the term $\pardiff[\vx]V(\vx)$ for all subsequent analysis since it is the gravity compensation control that is assumed to be readily available. We define the deterministic policy as:
\begin{align}
    \vu &=  \pi(\vx,\vxd;\eta,\psi,\phi) \nonumber\\
    &= -\pardiff[\vx]\left(\frac{1}{2}\vx^T\mathbf{S}(\eta)\vx + \Psi(\vx;\psi)\right)-\mathbf{D}(\vxd;\phi)\vxd \label{eqn:policy}\\
    &= -\mathbf{S}(\eta)\vx - \pardiff[\vx]\Psi(\vx;\psi)-\mathbf{D}(\vxd;\phi)\vxd \nonumber
\end{align}  
where $\mathbf{S}(\eta)\succ0$ contributes to a strongly convex potential function $\vx^T\mathbf{S}(\eta)\vx$, $\Psi(\vx;\psi)$ is a more flexible convex potential function and $\mathbf{D}(\vxd;\phi)$ is the damping function. Following the usual practice in energy shaping control, $\mathbf{D}(\vxd;\phi)$ does not depend on $\vx$ while $\Psi(\vx;\psi)$ depends only on $\vx$. We shall drop the parameter notations $\eta$, $\psi$ and $\phi$ whenever necessary in order to simplify notations. The deterministic policy structure shall be referred to as the energy shaping (ES) policy and is also illustrated in Fig. \ref{fig:arch}.

\begin{figure}[t]
    \centering
    \subfloat
       {\includegraphics[width=0.48\textwidth]{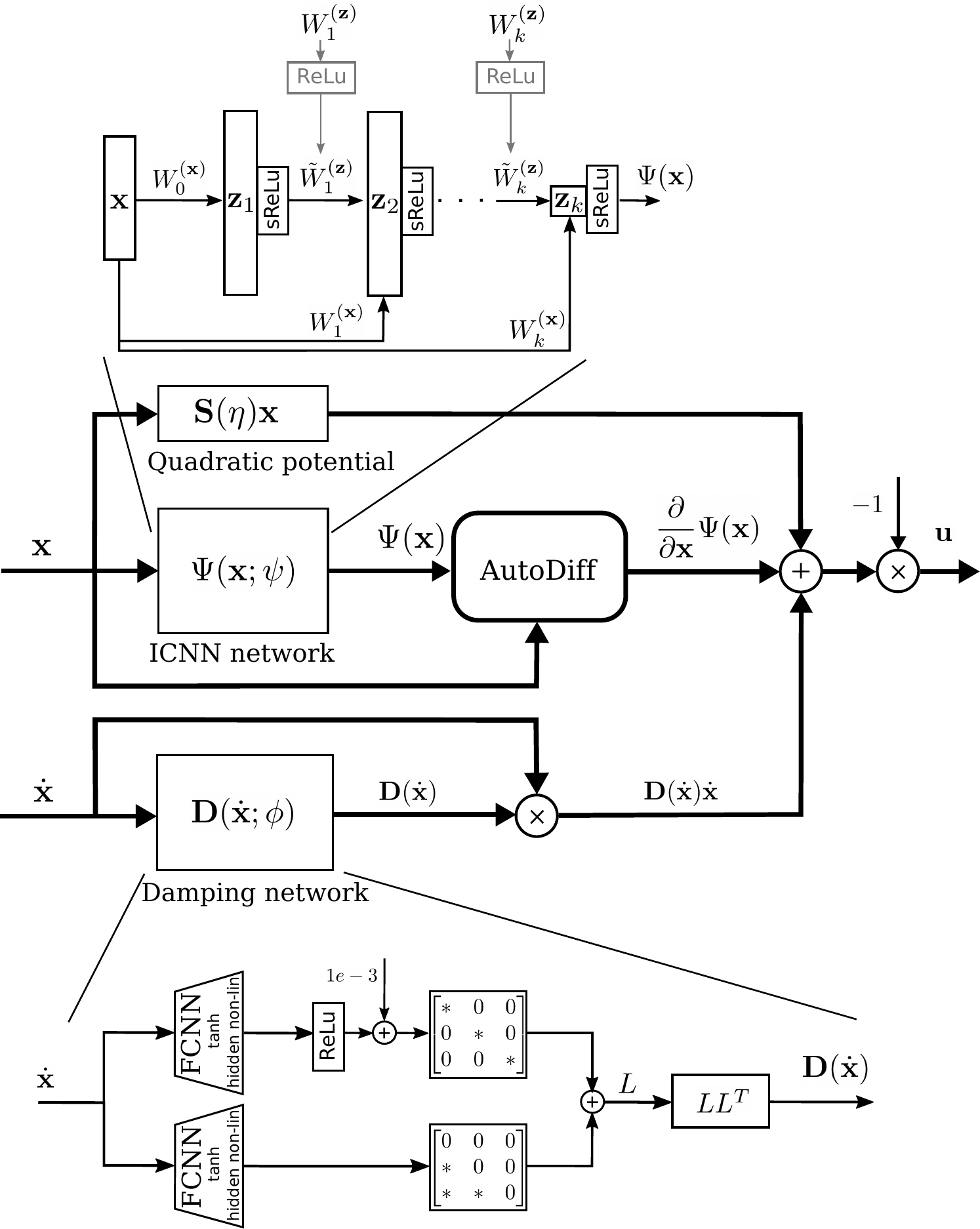}}\\
     \caption{\textbf{Network architecture for the energy shaping policy} }
     \label{fig:arch}
\end{figure}

The convex function $\Psi(\vx)$ need not have a unique minimum which then corresponds to non-asymptotic stability. If the task requires such a policy then the agent can learn a practically negligible $\mathbf{S}$ and an appropriate $\Psi(\vx)$ without a unique minimum. Alternatively, by learning a significant $\mathbf{S}$ in addition to $\Psi(\vx)$, the agent can achieve a strongly convex potential function with a unique minimum, one that corresponds to asymptotic stability. It is enough to show that $\Psi(\vx)$ is convex with minimum at $\vec{0}$ and $\mathbf{S}(\eta)\succ\vec{0}$ for any values of their respective parameters.

$\mathbf{S}(\eta)$ is defined as a diagonal matrix with the parameter $\eta$ as the diagonal elements. $\eta$ is clamped at lower and higher ends with positive numbers to ensure positive definiteness and to limit the contribution of the quadratic function.

We propose to implement $\Psi(\vx)$ using the FICNN version of ICNN \cite{amos2017input}. ICNN has been used previously to model a Lyapunov function in \cite{kolter2019learning}. In our case, the ICNN potential function will only be a part of a Lyapunov function. The non-negativity of $W_i^{(\vec{z})}$ in Eq. (\ref{eqn:ficnn}) is implemented by applying the ReLu function on the parameter itself. The convexity of $g_i$ in Eq. (\ref{eqn:ficnn}) is satisfied by adopting the smooth variant of ReLu (sReLu), introduced in \cite{kolter2019learning}, which is convex, non negative and differentiable. Differentiability of activation function is required to ensure that $\Psi(\vx)$ is continuously differentiable or class $C^1$. Both aspects are shown in Fig. \ref{fig:arch}. Unlike the intended use of FICNN, our requirement is such that $\Psi(\vx)$ should have a minimum at $\vec{0}$. For any random initialization of FICNN this is not necessarily true. We could find the minimum through any gradient based optimization and add offsets to both $\vx$ and $\Psi$ to give $\Psi(\vec{0})=0$. However, we adopt a less expensive method, one that avoids any numerical optimization, by simply setting $b_i=\vec{0}$ in Eq. (\ref{eqn:ficnn}). With $\vec{z}_0, W_0^{(\vec{z})}\equiv\vec{0}$ and $\vx=\vec{0}$: 
\begin{equation}
    \vec{z}_{i+1} = g_i(W_i^{(\vec{z})}\vec{z}_i + W_i^{(\vec{x})}\vx) = \vec{0} \quad \textit{for}\ i=0...k-1 
\end{equation}
Together with the fact that $\Psi(\vx)\ge0$, due to the sReLu activation function, it means that $\vx=\vec{0}$ is a minimum with $\Psi(\vec{0})=0$. By considering the combined potential function in Eq. (\ref{eqn:policy}), we also have $(\vx^T\mathbf{S}\vx + \Psi(\vx))>0\ \forall\vx\ne\vec{0}$. Finally, as shown in Fig. \ref{fig:arch}, the gradient $-\pardiff[\vx]\Psi(\vx)$ is obtained readily by automatic differentiation. This is an instance of inference time automatic differentiation as opposed to the usual case of training time automatic differentiation.

The damping function in Eq. (\ref{eqn:policy}) is structured as shown in Fig. \ref{fig:arch}. Two independent fully connected networks (FCNN) with \textit{tanh} hidden layer non-linearity output the elements of the damping matrix. The first and second networks output the diagonal and the off-diagonal elements, respectively. The diagonal elements are first made non negative through a ReLu activation and then positive by adding a small positive number. The output layer of the second FCNN has no non linearity. Together, the two outputs combine to form a lower triangular matrix with positive diagonal elements. The product of the lower triangular matrix with itself will produce the required positive definite damping matrix $\mathbf{D}(\vxd)$. A similar strategy was used previously for a symmetric positive definite inertia matrix in \cite{lutter2019iclr}. The two networks may also be combined into one if that is preferable. 

\subsection{Stability of the Energy Shaping Policy}\label{sct:methods:stability_proof} 
In the following theorem, we show that the proposed policy parameterization is automatically stable and passive independent of the network parameter values. 
\begin{theorem}\label{theo:stability}
Let $\mathbf{S}$ be symmetric positive definite, $\Psi(\vx)$ be convex in $\vx$, a generalized coordinate, with minimum at $\vec{0}$ and $\Psi(\vec{0})=0$, and the function $\mathbf{D}(\vxd)$ be also positive definite. Then a gravity compensated manipulator under the control of the deterministic policy $\vu = -\mathbf{S}\vx - \pardiff[\vx]\Psi(\vx)-\mathbf{D}(\vxd)\vxd$ is: (i) globally asymptotically stable at $\vx=\vec{0}$ if $\vec{f}_{ext}=\vec{0}$ and (ii) passive with respect to $\vec{f}_{ext}$ if $\vec{f}_{ext}\ne\vec{0}$. 
\end{theorem}
\begin{proof}
Consider the Lyapunov function
\begin{align*}
    V_L(\mathbf{x}, \Dot{\mathbf{x}}) = \frac{1}{2}\vx^T\mathbf{S}\vx + \Psi(\vx)+\frac{1}{2}\Dot{\mathbf{x}}^T\mathbf{M}(\mathbf{x})\Dot{\mathbf{x}} \nonumber
\end{align*}
where $\mathbf{M}(\mathbf{x})$ is the inertia matrix. We already have $V_L(\vx,\vxd)$ in class $C^1$, $V_L(\vec{0},\vec{0})=0$, $V_L(\vx,\vxd)>0\ \forall\vx\ne\vec{0},\vxd\ne\vec{0}$ and as $||\vx||,||\vxd|| \rightarrow \infty \implies V_L(\vx,\vxd)\rightarrow\infty$. The first three properties are achieved by design as explained earlier. The last property follows from the fact that for large values of $\vx$ and $\vxd$, $V_L(\vx,\vxd)$ behaves like a quadratic function because $\Psi(\vx)$ only grows linearly thanks to the sReLu activation.

Taking the time derivative, 
\begin{align}
    \Dot{V}_L = \vxd^T\mathbf{S}\vx + \vxd^T\pardiff[\vx]\Psi(\vx) + \Dot{\mathbf{x}}^T\mathbf{M}(\mathbf{x})\Ddot{\mathbf{x}} + \frac{1}{2}\Dot{\mathbf{x}}^T\Dot{\mathbf{M}}(\mathbf{x})\Dot{\mathbf{x}}\nonumber
\end{align}
Substituting $\mathbf{M}(\mathbf{x})\Ddot{\mathbf{x}}$ from Eq. (\ref{eqn:manipultor}), substituting $\vu$ from Eq. (\ref{eqn:policy}), utilizing the skew-symmetric property of $\Dot{\mathbf{M}}-2\mathbf{C}$ and utilizing positive definiteness of $\mathbf{D}(\vxd)$ leads to the last required property for global asymptotic stability 
\begin{align*}
    \Dot{V}_L  = -\Dot{\mathbf{x}}^T \mathbf{D}(\vxd) \Dot{\mathbf{x}} < 0 \quad (\vec{f}_{ext}=\vec{0})
\end{align*}
When considering $\vec{f}_{ext}\ne\vec{0}$, 
\begin{align*}
    \Dot{V}_L  = \vxd^T\vec{f}_{ext} -\Dot{\mathbf{x}}^T \mathbf{D}(\vxd) \Dot{\mathbf{x}} < \vxd^T\vec{f}_{ext}
\end{align*}
which fulfills the condition for passivity that requires $\Dot{V}_L $ to be upper bounded by the flow of energy from external sources.
\end{proof}

\textit{Remark:} A Cartesian space control version of the ES policy can be implemented based on the Jacobian transpose method presented in Sections 8.6.2 and 9.2.2 of \cite{siciliano2010robotics}. In the case of redundant manipulators, where the Cartesian space is not a generalized coordinate system, stability can be preserved by adopting a slightly reformulated Lyapunov function that considers joint space kinetic energy term and joint space viscous friction. More details can be found in the above references.

\subsection{Stability-Guaranteed RL with Energy Shaping Policy}
In Sec. \ref{sct:methods:stability_proof}, we showed that our policy parameterization leads to unconditional global asymptotic stability for free motion and passivity when acted upon by an external force. When the external force is a result of physical interaction with passive objects in the environment, the overall system is also stable \cite{hogan2018impedance, colgate1988robust}. Passive objects are those that cannot generate energy and are the most common cases (see Sec. \ref{sct:disc}). Note that asymptotic stability may be lost due to motion constrained by obstacles but stability is still guaranteed. We have achieved the goal regarding policy structuring since our policy guarantees stability for potentially contact-rich tasks without any conditions or constraints for policy initialization or search. The policy search is also model-free since it does not require any learned dynamics model.

With the deep policy in place, we only require a suitable policy search algorithm. The ES policy is a deterministic policy and injecting noise into its action space would invalidate the stability proof in Theorem \ref{theo:stability}. Fortunately, it is unconditionally stable for any network parameter value and is, therefore, suitable for unconstrained parameter space exploration strategies. Since CEM (see Sec. \ref{sct:prelim:cem}) is one such strategy, we formulate a CEM based RL with the state and action defined as $\vs=[\vx^T\ \vxd^T]^T$ and $\va=\vu$, respectively, the parameter set $\theta=\{\eta, \psi, \phi\}$, and a Gaussian with diagonal covariance for the sampling distribution $q(\theta|\boldsymbol{\Phi})$.

%%%%%%%%%%%%%%%%%%%%%%%%%%%%%%%%%%%%%%%%%%%%%%%%%%%%%%%%%%%%%%%%%%%%%%%%%%%%%%%%

\section{Experimental Results}

\begin{figure}[t]
    \centering
    \subfloat
       {\includegraphics[width=0.35\textwidth]{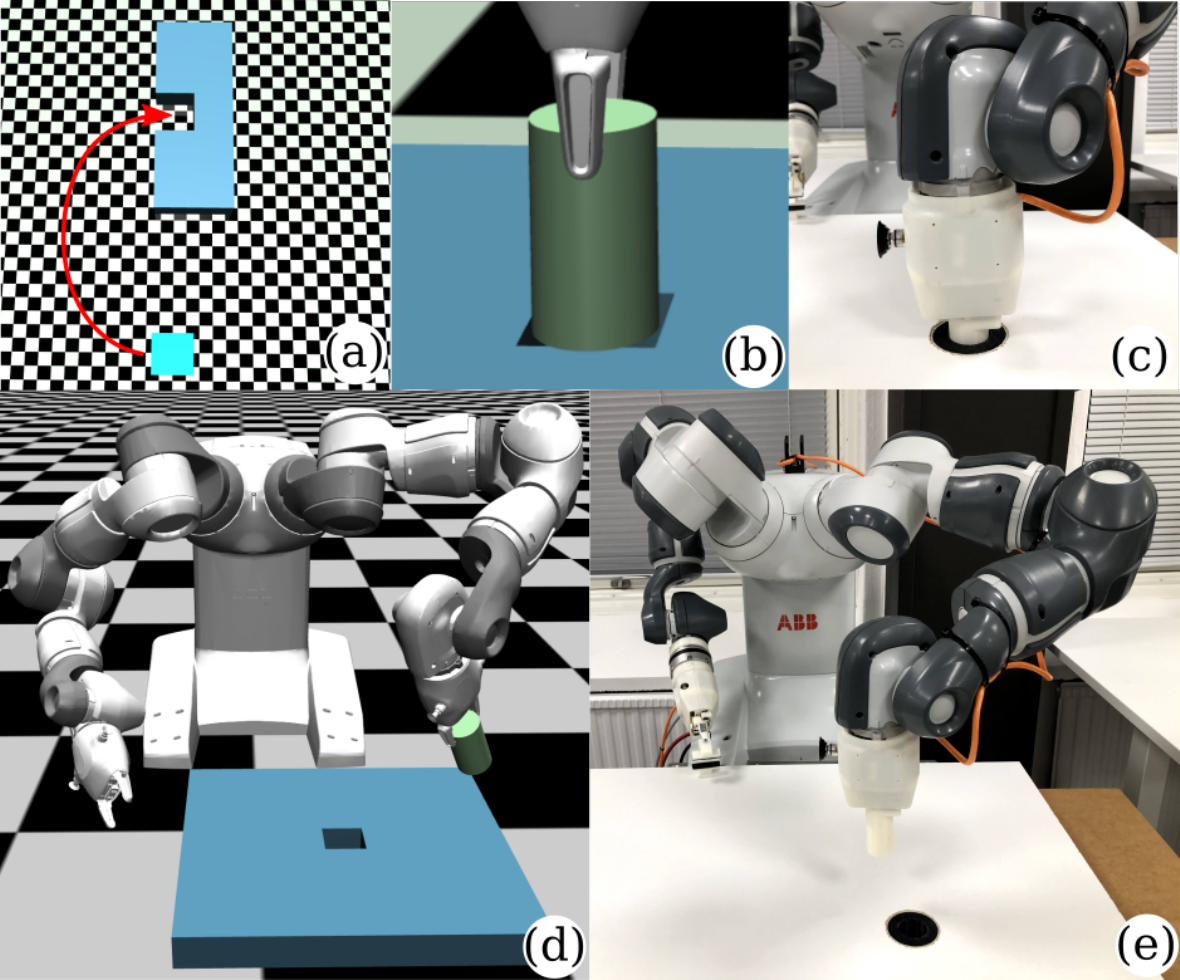}}\\
     \caption{\textbf{Experimental setup} \textbf{(a)} 2D Block-insertion: a cubical block of side $50$ mm is to be inserted into a slot with a clearance of 2 mm. The initial position and an illustrative path is also shown. \textbf{(b \& d)} Simulated peg-in-hole: a cylindrical peg of diameter of $23$ mm and a square hole of width and depth of $25$ mm and $40$ mm, respectively. \textbf{(c \& e)} Physical peg-in-hole: a cylindrical peg of diameter $27$ mm and a cylindrical hole of width and depth of $28$ mm and $40$ mm, respectively. An insertion depth of $25$ mm and $35$ mm are considered successful for the 2D block-insertion and the robotic insertion tasks, respectively.}
     \label{fig:exp_setup}
\end{figure}

For validating our method, we focus on contact-rich manipulation tasks. Contact-rich manipulation is relevant for robotics because it may be thought of as representing the superset of robot motion problems. This is because it covers both the path generation and the force interaction aspects of manipulation. Contact-rich manipulation is usually represented by various insertion tasks. A simplified representative problem is the so called peg-in-hole problem~\cite{yun2008compliant,lee2019making}. 

Adopting the experimental setups from \cite{stableRAL2020khader}, we start with a simulated 2D block-insertion task without a manipulator (Fig. \ref{fig:exp_setup}a), scale up to a 7-DOF manipulator peg-in-hole (Fig. \ref{fig:exp_setup}b,d) and then move to a physical version of peg-in-hole (Fig. \ref{fig:exp_setup}c,e). We use the MuJoCo physics simulator and a physical YuMi robot for the experiments. The proposed method will be referred to as energy shaping with CEM (ES-CEM). For baseline comparison, we use a standard deep policy with Proximal Policy Optimization \cite{ppo2017shulman} (ANN-PPO) and the normalizing flow control policy \cite{khader2020learning} with CEM (NF-CEM). We use the PyTorch version of the RL software framework \textit{garage} \cite{garage}. All simulation experiments are done with gravity compensated and five random seeds. The resulting variances are incorporated into the error bars in the plots.  

NF-CEM is the main baseline since it also has stability properties. The main idea here is to learn a complex coordinate transformation function, through RL, such that a fixed spring-damper regulator in the transformed coordinate space can solve the task. Despite the presence of the deep network, control stability is preserved. The spring-damper is idealized to be 'normal', defined as identity matrices for the stiffness ($K$) and damping ($D$) gains. 
% As a general principle, we repeat each experiment with several random seeds and report only representative cases.

Values of relevant hyperparameters are: number of rollouts or CEM samples per iteration $N_s= 15$; number of elite samples in CEM $N_e=3$; episode length $T=200$; initial value of $\eta$ is $0.1$; clipping values for $\eta$ are $\eta_{min}=1e-3$ and $\eta_{max}=5$; initial standard deviation for $q(\theta|\boldsymbol{\Phi})$, or $\boldsymbol{\Phi}_{\sigma}^0$, are $0.2$ and $0.15$ for ES-CEM and NF-CEM, respectively; initial action standard deviation for PPO is $2.0$; number of flow elements in NF policy $n_{flow}=2$; and dimension of flow element layers in NF policy $n_{h\_flow}=16$. Initial mean for $q(\theta|\boldsymbol{\Phi})$, or $\boldsymbol{\Phi}_{\mu}^0$, is obtained by randomly initializing all network parameters with the Pytorch functions \textit{xavier\_uniform\_} and \textit{zeros\_} for weights and biases, respectively. We used the default settings in \textit{garage} for PPO. The reward function presented in~\cite{levine2015learning} is used for all our experiments.

\subsection{2D Block-Insertion}

 \begin{figure}[t]
    \centering
    \subfloat
       {\includegraphics[width=0.49\textwidth]{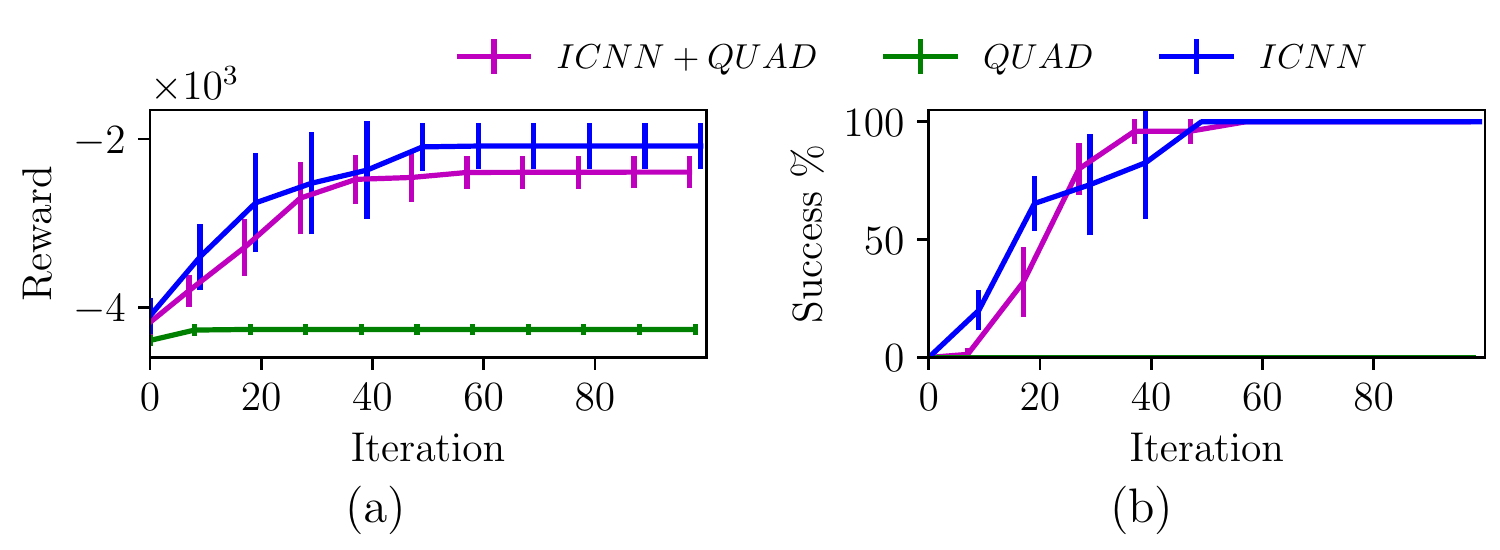}}\\
       \vspace{-2.5mm}
     \subfloat{\includegraphics[width=0.49\textwidth]{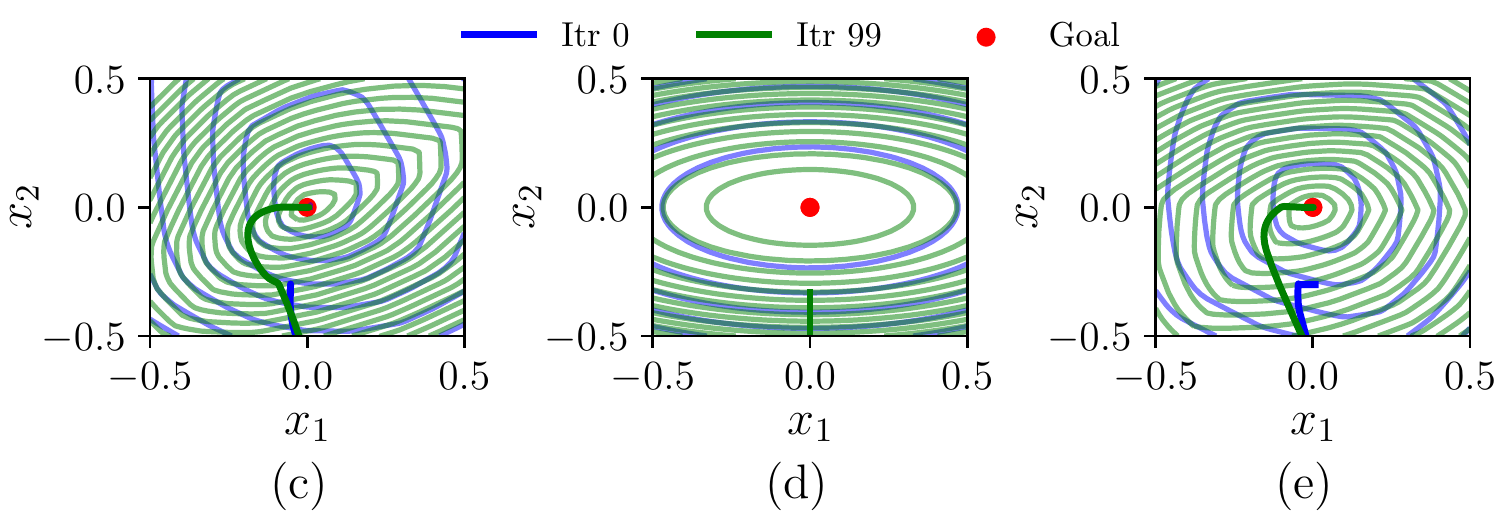}}
     \vspace{-3mm}
     \caption{\textbf{Potential Function Parameterizations for 2D Block-insertion:} RL experiments with potential functions of ICNN, quadratic function and ICNN plus quadratic function. \textbf{(a)} Learning curve \textbf{(b)} Success rate of insertions. Contour plots of potential functions with its corresponding trajectory overlaid: \textbf{(c)} ICNN, \textbf{(d)} quadratic function and \textbf{(e)} ICNN plus quadratic function. Blue corresponds to a parameter sample from iteration 0 (randomly initialized parameter distribution) and green corresponds to a parameter sample from the last iteration.}
     \label{fig:block2D_icnn_ablation}
\end{figure}

A block of mass that is permitted to only slide on a surface but not rotate is to be inserted into a slot with a small clearance. The block is controlled by applying a pair of orthogonal forces at its center of mass. The initial position is shown in Fig. \ref{fig:exp_setup}a along with an illustrative path. We used hidden layer sizes [16,16] for ICNN, [8,8] for each of the FCNN damping modules and [16,16] for ANN-PPO. 

Since we are primarily interested in the modeling capacity of ICNN, we conduct an ablation study with only ICNN, only quadratic function and the combination of both. The damping network part remains the same. In Fig. \ref{fig:block2D_icnn_ablation}a-b we see that ICNN by itself is capable of solving the task while the pure quadratic case is not. The combination model is necessary for our theoretical stability proof to hold. From Fig. \ref{fig:block2D_icnn_ablation}c-e we can see the potential function contour plots for sample cases from first and last iterations. The quadratic potential is clearly incapable of moving the block in curved paths. Even for the other two cases, the initial iteration samples are always unsuccessful.

 \begin{figure}[t]
    \centering
    \subfloat
       {\includegraphics[width=0.49\textwidth]{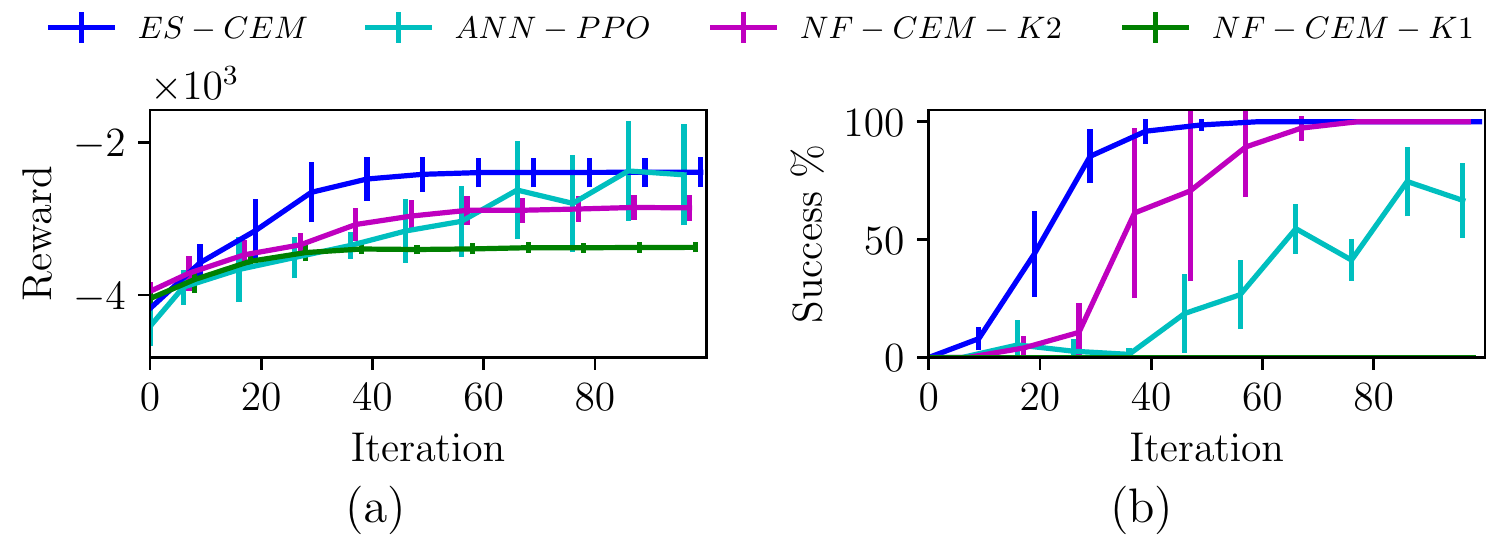}}\\
       \vspace{-4mm}
     \subfloat{\includegraphics[width=0.49\textwidth]{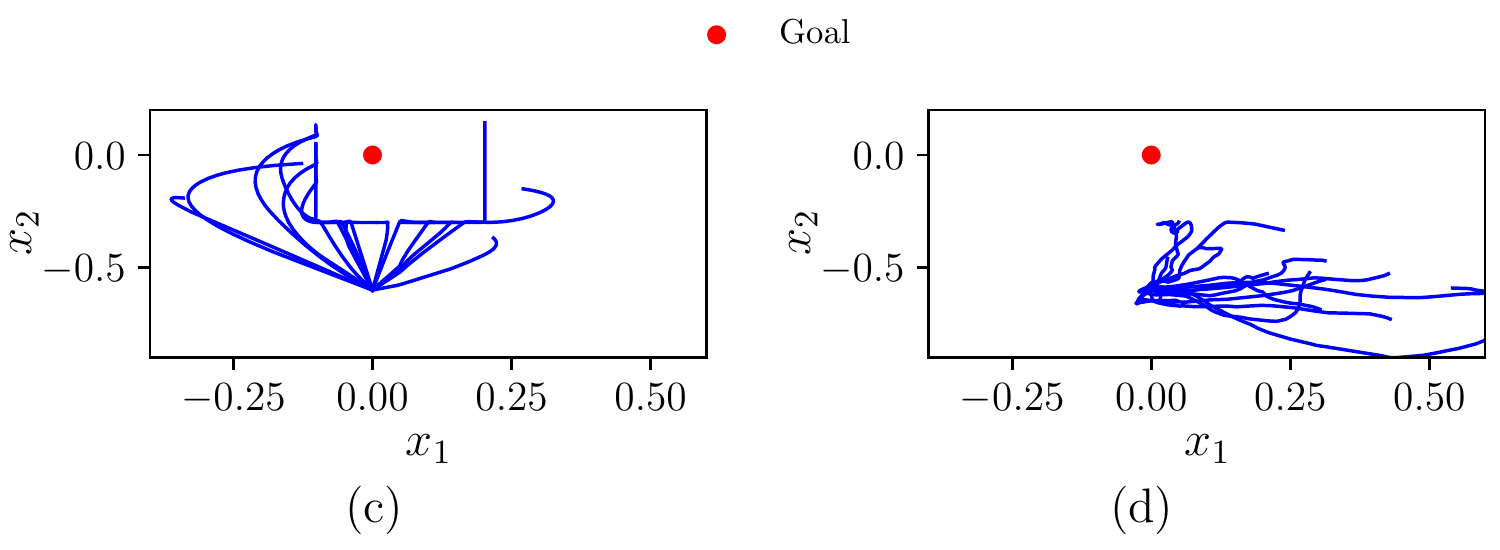}}
     \vspace{-3mm}
     \caption{\textbf{Reinforcement Learning, 2D Block-insertion:} RL results of ES-CEM along with that of the baselines. \textbf{(a)} Learning curve \textbf{(b)} Success rate of insertions. \textbf{(c)} Smooth trajectory rollouts of ES-CEM that tend towards the goal even in the first iteration. \textbf{(d)} Noisy and divergent trajectory rollouts of the standard ANN policy (action space exploration) in the first iteration.}
     \label{fig:block2D_rl}
\end{figure}

The NF policy turned out to be dependent on the latent space spring-damper system (see Fig. \ref{fig:block2D_rl}a-b). It failed to solve the task for the default value of the stiffness parameter $K=I$. However, NF-CEM succeeds for $K=2I$. We conclude that NF policy is competitive to ES policy but incurs an additional burden of extra task sensitive hyperparameters. We also compare with the ANN-PPO as a standard baseline. In Fig. \ref{fig:block2D_rl}a-b we see that ANN-PPO has the worst performance. This is consistent with results in \cite{khader2020learning} for a simpler version of the task. Following \cite{khader2020learning} and \cite{stableRAL2020khader}, we attribute the better performance of ES-CEM and NF-CEM to the ability of stable policies for forming attractors to the goal. Finally, we showcase the strength of our approach with respect to action space exploration. Fig. \ref{fig:block2D_rl}c-d show that even when randomly initialized, the trajectories induced by the ES policy tend towards the goal. The trajectories are also smooth owing to parameter space exploration.

\subsection{Simulated Peg-In-Hole}
 \begin{figure}[t]
    \centering
    \subfloat
       {\includegraphics[width=0.49\textwidth]{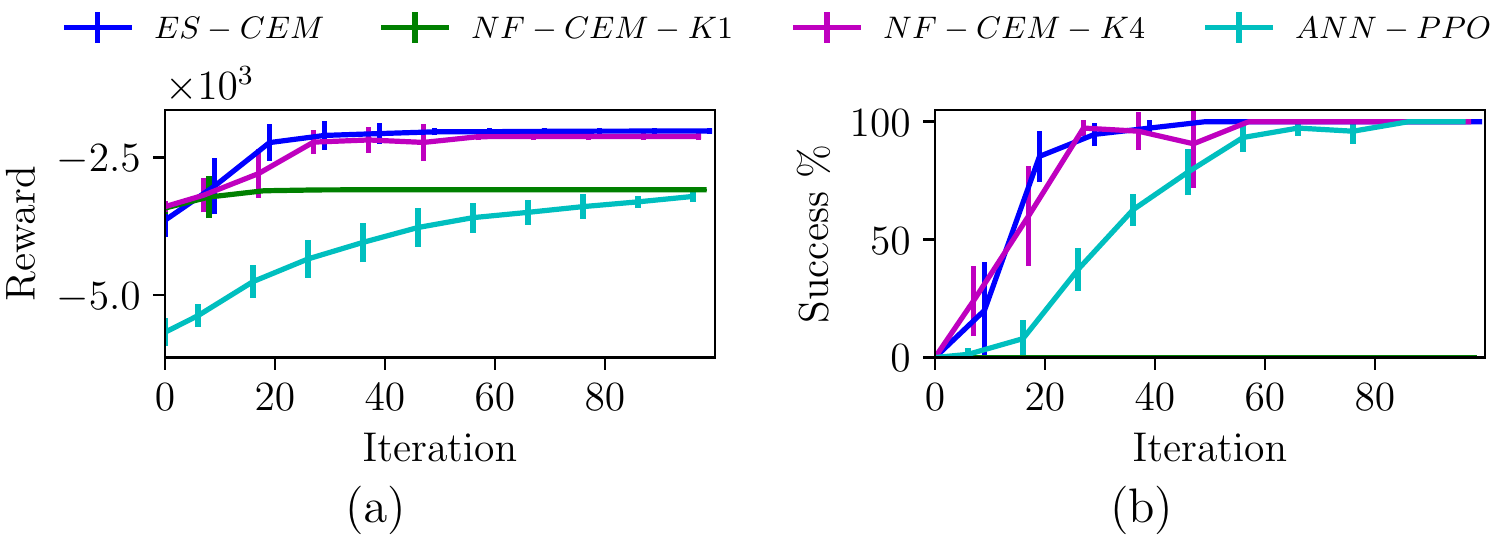}}\\
       \vspace{-3mm}
     \caption{\textbf{Reinforcement Learning, peg-in-hole:} RL results of ES-CEM along with that of the baselines. \textbf{(a)} Learning curve \textbf{(b)} Success rate of insertions.}
     \label{fig:peg_rl}
\end{figure}

A simulated 7-DOF manipulator arm is expected to learn to insert a rigidly held peg into a hole with low clearance. To simplify, only the translation part of the Cartesian space is considered for RL while the rotation part is controlled with a stable impedance controller (PD regulator) that regulates towards the goal orientation. Our method can be extended to include rotation by adopting the same formulation as \cite{stableRAL2020khader}. The initial position is a fixed but randomly generated configuration as shown in Fig. \ref{fig:exp_setup}d. We use hidden layer sizes of [24,24] for ICNN, [12,12] for each of the FCNN modules in the damper network and [32,32] for the standard ANN policy in the ANN-PPO baseline.

Fig. \ref{fig:peg_rl} also leads to the same observation that the NF policy is competitive to the ES policy only when its latent space spring-damper system is configured appropriately ($K=4I$). Such tuning was not necessary with the NF-PPO setup in \cite{khader2020learning} perhaps due to larger state space coverage. ANN-PPO again performed worse than the rest due to the same reason mentioned earlier. We notice that its performance is considerably better for the peg-in-hole task than the simpler block-insertion task. We attribute this to the compliant behaviour in the rotation space compared to the perfectly rigid interaction in the block-insertion task. 
% The attentive reader may also have noticed the significant shift in the ANN-PPO learning curve compared to the others. This can be attributed to the stable policy causing the manipulator to reach the vicinity of the goal earlier and thus leaving more time steps to collect rewards. Such an opportunity was reduced by the environmental obstacles in the block-insertion case.

\subsection{Physical Peg-In-Hole}

 \begin{figure}[t]
    \centering
    \subfloat
       {\includegraphics[width=0.45\textwidth]{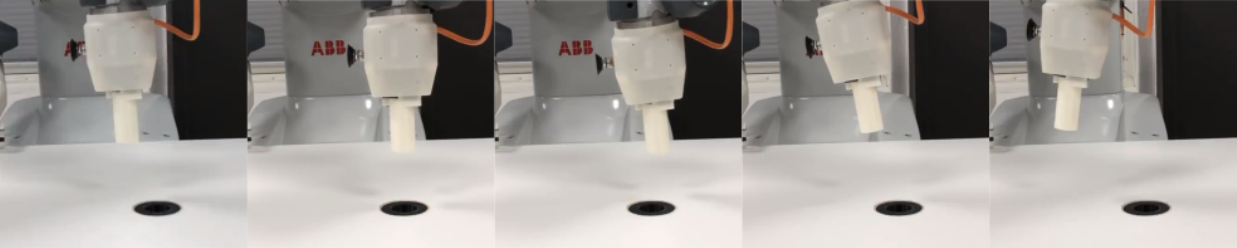}}\\
     \caption{\textbf{Samples of initial position:} drawn with mean of the training time position and standard deviation $\sigma_{\Vec{q}_0}=0.05$. The leftmost is the training position.}
     \label{fig:robustness_init_pos}
\end{figure}

 \begin{figure}[t]
 \vspace{-3mm}
    \centering
    \subfloat
       {\includegraphics[width=0.49\textwidth]{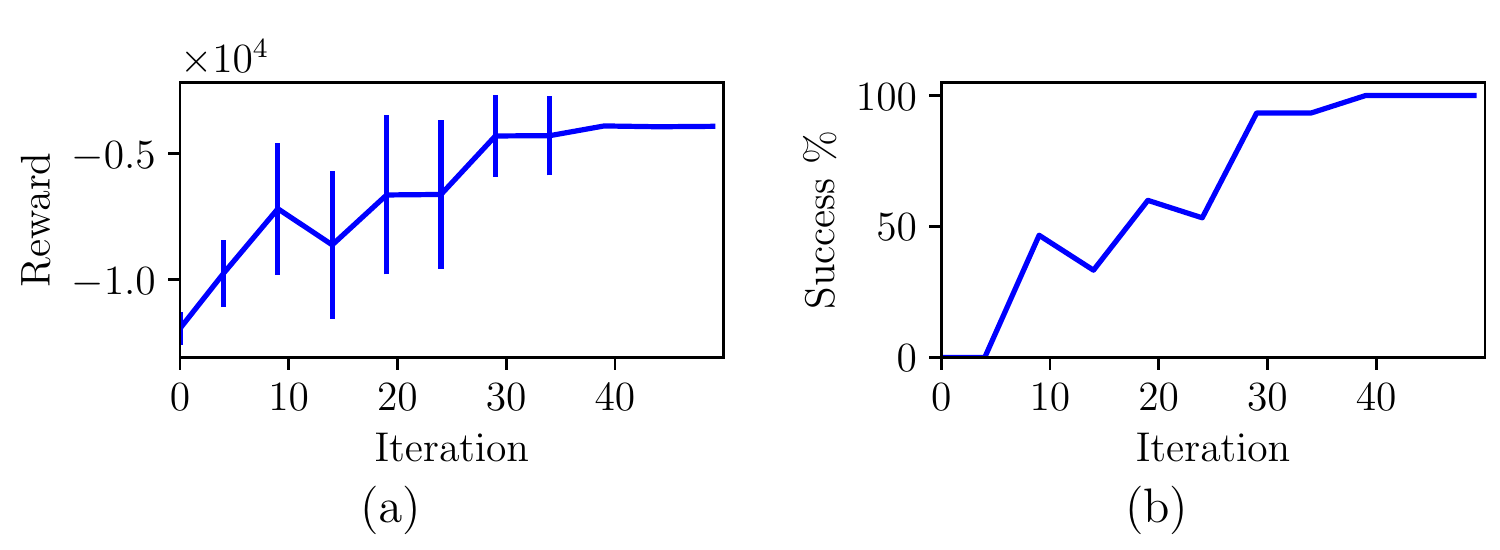}}\\
       \vspace{-3mm}
     \caption{\textbf{Reinforcement Learning, physical peg-in-hole:} RL results of ES-CEM. \textbf{(a)} Learning curve \textbf{(b)} Success rate of insertions.}
     \label{fig:peg_real_rl}
\end{figure}

\begin{table}[t]
 \caption{Generalization to Initial Position}
 \label{tab:robust_init_q}
\centering
\begin{tabular}{ |c|c|c|c|c| } 
 \hline
 $\sigma_{\Vec{q}_0}$ (rad) & $0.02$  & $0.04$ & $0.05$ & $0.1$ \\ 
 \hline
 Success \% & $100$  & $80$ & $60$ & $33.3$ \\ 
  \hline
\end{tabular}
\end{table}

In this experiment, we realize a physical version of the peg-in-hole task and evaluate only the ES-CEM case. The initial position is as shown in Fig. \ref{fig:exp_setup}e. We use identical settings and hyperparameters as of the simulated case except $\boldsymbol{\Phi}_{\sigma}^0=0.3$, $T=500$, and an insertion clearance of $1$ mm. The insertion clearance was lowered from $2$ mm to $1$ mm to increase the difficulty of the problem. Similar to the simulation case, the robot exhibited stable behaviour with trajectories tending towards the goal right from the beginning. 

In Table \ref{tab:robust_init_q}, we report the success rate for different initial state distributions after training. We conducted a set of $15$ trials for each of four initial position ($\Vec{q}_0$) distributions, whose means were the fixed position used for training and element-wise standard deviations ($\sigma_{\Vec{q}_0}$) as shown in the table. The distributions are defined for the joint position coordinates. Figure \ref{fig:robustness_init_pos} shows that the distributions did induce enough spread in the Cartesian space translation component to challenge low clearance insertion. The observed generalization to unseen initial position can be attributed to the policy's global stability.

% From this experiment, it was evident that not only is there a potential to successfully generalize to unseen initial positions, but even for failure cases the trajectories were has reasonable robustness  shows that the learned policy, owing to the stability property, is able to generalize to 

\section{Discussions and Conclusions} \label{sct:disc}
In this paper, we investigated the possibility of achieving stability-guaranteed DRL for robotic manipulation tasks. We formulated a strategy of 1) structuring a deep policy with unconditional stability guarantee, and 2) adopting a model-free policy search that preserves stability. The first was achieved by deriving a deep policy structure from energy shaping control of Lagrangian mechanics. The second was realized by adopting the well-known parameter space search algorithm CEM. Our method is model-free as long as the manipulator is gravity compensated. Neither the stability proof nor the policy search requires the prior knowledge or learning of any model. Moreover, stability is preserved even when interacting with unknown passive environments. The method was validated on simulated and physical contact-rich tasks. Our result is notable because no prior work has demonstrated, to the best of our knowledge, a stability-guaranteed DRL on a robotic manipulator. 

The main benefit of our work is that it proposes one way of moving away from analytic shallow policies to expressive deep policies while guaranteeing stability. In addition to improving predictability and safety, stability also has the potential to reduce sample complexity in RL \cite{khader2020learning}. This is due to the fact that the goal position information is directly incorporated into the policy. In fact, due to inherent stability of the policy, it is possible to deploy state-of-the-art policy gradient algorithms such PPO instead of CEM. Although such an approach will invalidate the formal stability proof, the user will still enjoy the practical benefit of stability as was demonstrated on the NF policy in \cite{khader2020learning}. Else, more sophisticated parameter space exploration strategies such as CMA-ES \cite{hansen2001completely} and NES \cite{wierstra2014natural, salimans2017evolution} are also an option to scale up without invalidating the stability proof. Another possibility to exploit the stability property would be in a sparse reward setting since convergence behavior towards the goal can potentially minimize reward shaping.

% , in the spirit of \cite{Plappert18Parameter}.
% Additionally, gradient based parameter update may also better handle the sensitivity of the output actions to parameter perturbations, something that forced us to choose an appropriate initial variance for the parameter distributions.

The stability property of the proposed ES policy is shared by the previously proposed NF policy structure \cite{khader2020learning}, although the latter is not completely a deep policy. Our results indicate that the practical learning outcome of the NF policy is dependent on task specific initialization of its unlearned spring-damper system. The proposed ES policy has no such hyperparameters. The NF policy will come in handy when the spring-damper values can positively bias the policy and thus speed up the learning process, whereas the ES policy has no such mechanism to incorporate a bias.

An important consideration is the model capacity of the ES policy. While an ANN policy has the potential to approximate any nonlinear function, the same cannot be said about the ES policy. Nevertheless, it may be noted that the ES policy can potentially represent any policy within the energy shaping control form for a given fully actuated Lagrangian system. This is because the ICNN can represent any convex function \cite{amos2017input} and the damping network, parameterized by fully-connected network modules, could represent all nonlinear damping functions. As an alternative to scaling up the ICNN network, combining multiple convex functions through any of the convexity preserving operations can also be used to construct complex convex functions. However, activating a number of convex functions in sequence is not supported since that would require more sophisticated stability analysis techniques. The energy shaping control structure may very well be more limiting than an arbitrary nonlinear control law. A theoretical analysis of that nature is outside the scope of this work. Our empirical results show that the ES policy is expressive enough to learn complex contact-rich tasks. For example, the block insertion result shows that the ES policy can represent curved paths and also the necessary interaction behavior that regulates exchange of forces and exploit contacts. 

Other than the possible limitations in model capacity, few other limitations are also worth mentioning: i) If the setting of the bias terms in ICNN to zero is to be avoided, as a measure to enhance representational capacity, one would need to introduce a convex optimization step to discover the minimum of the ICNN (see Sec. \ref{sct:policy_param}). This is just an additional step and will not jeopardize the RL process. ii) The ES policy is only valid for fully actuated Lagrangian systems. This is not a serious limitation because all serial link manipulators are fully actuated. iii) The ES policy is valid only for episodic tasks with a unique goal position, thus excluding cyclic or rhythmic tasks. iv) The stability analysis can be extended to the rotational space only for the Euler angle representation. A straightforward extension can be done according to \cite{stableRAL2020khader}. Lastly, v) stability is preserved only for passive environments. This is also not a serious limitation because most objects in the environment are unactuated and hence passive. Notable exceptions are other robots, actuated mechanisms and human users.

Stability guarantee in DRL cannot be overstated. It is one of the main means to achieve safety in DRL. Although many recent works have made significant contributions in this field, very rarely do they consider a robotic manipulator interacting with the environment. As a result, most of these methods are either not suitable or do not scale well to the manipulation case. Our work can be seen as an attempt to bridge this gap.



% \section*{ACKNOWLEDGMENT}
% This work was partially supported by the Wallenberg AI, Autonomous Systems and Software Program (WASP) funded by the Knut and Alice Wallenberg Foundation.

%%%%%%%%%%%%%%%%%%%%%%%%%%%%%%%%%%%%%%%%%%%%%%%%%%%%%%%%%%%%%%%%%%%%%%%%%%%%%%%%

\bibliographystyle{IEEEtran}
\bibliography{references/rl,references/rl_skill,references/rl_skill_compliant,references/other,references/imitation_learning,references/control_opt_robotics,references/model_learning,references/ml}

\end{document}